\documentclass{article}   
\usepackage{ijcai15} 
\usepackage{hieroglf}
\newcommand{\hide}[1]{}    
\usepackage{rotate}    
\usepackage{graphicx} 
\usepackage{tabularx}
\usepackage[utf8]{inputenc} 
\usepackage{listings}  
\usepackage{amsthm} 
\usepackage{amsmath}
\usepackage{amssymb}
\usepackage{textcomp}   
\usepackage{multicol}
\usepackage{enumitem}
\usepackage{bussproofs}
\usepackage{color} 
\usepackage{times} 
\usepackage{tikz}    
\usetikzlibrary{lindenmayersystems}
\usepackage[framemethod=TikZ]{mdframed}
\usepackage{marvosym} 
\usepackage{listings}
\usepackage{caption}   
\usepackage{multirow} 
\usepackage{wrapfig}   
\usepackage[sans]{dsfont}
\usepackage{verbatim}
\usepackage[lofdepth,lotdepth]{subfig} 
     
\newtheorem{observation}{Observation}

\newtheorem{theorem}{Theorem}

\newcommand{\btriplet}{\textsf{BTriplet}}
\newcommand{\cond}{\textsf{Cond}}
\newcommand{\conflict}{\textsf{Conflict}}
\newcommand{\logcon}{\textsf{L}}

\newcommand{\bbase}{\textsf{BBase}} 
\newcommand{\cna}{\ensuremath{\hat{Cn}}} 
\newcommand{\xa}{\ensuremath{\hat{B}}} 
\newcommand{\diva}{\ensuremath{\hat{\div}}} 
\newcommand{\stara}{\ensuremath{\hat{*}}} 
\newcommand{\plusa}{\ensuremath{\hat{+}}} 
\newcommand{\qa}{\ensuremath{\hat{P}}}
   
\newtheorem{definition}{Definition}

\newcommand{\assoc}{\textsf{Assoc}}
 
\newcommand{\lit}{\textsf{Lit}} 
\newcommand{\exclude}{\textsf{Exc}}
\newcommand{\props}{\textsf{Props}}

\newcommand{\andMeta}{\ensuremath{\wedge^{\dagger}}}
\newcommand{\orMeta}{\ensuremath{\vee^{\dagger}}}

\newcommand{\visible}{\textsf{Visible}}

\title{How do you revise your belief set with 
     \%\$£;@*?} 

\begin{document} 
\maketitle
    \begin{abstract}          
       In the classic AGM belief revision theory, 
       beliefs are static and do not change 
       their own shape. 
       For instance, if 
       $p$ is accepted by a rational agent, it will remain $p$ 
       to the agent. But such rarely happens to us. 
       Often, 
       when we accept some information $p$, 
       what is actually accepted is not the whole 
       $p$, but only a view of it; 
       not necessarily because we select the portion but because 
       $p$ must be perceived. Only the perceived 
       $p$ is accepted; and 
       the perception  
       is subject to what we already believe (know). 
       What may, however,  happen 
       to the rest of $p$ that initially escaped our attention?  
       In this work we argue that the invisible 
       part is also accepted to the agent, if 
       only unconsciously. Hence 
       some parts of $p$ are accepted as visible, 
       while some other parts as latent, beliefs. 
       The division is not static. As 
       the set of beliefs changes, what were hidden 
       may become visible. 
       We present a perception-based belief theory that 
       incorporates latent beliefs. 
    \end{abstract} 
\section{Introduction}        
In the classic AGM belief revision theory \cite{Makinson85}, 
what a rational agent 
is committed to believe \cite{Levi91} forms 
his/her belief set X of formal sentences, which is usually assumed 
consistent and closed under logical consequences. 
To $X$ 
an input $P$, again some (hopefully consistent) sentence, is passed. 
If $P$ is not in conflict with beliefs in $X$, 
it is simply incorporated into $X$ without 
any prior operations. If not, however, 
the simplistic augmentation leads to inconsistency. 
In such situations, minimal changes 
are made to $X$ beforehand. Then the result 
is a consistent set. 
Either way, 
the new belief set is postulated to include $P$ (success postulate), to be consistent unless $P$ itself is inconsistent, 
and to be closed under logical consequences. 
This is roughly what it takes in the AGM belief 
revision. \\
\indent Some of the subsequent works have felt that there 
are parts in the AGM theory that may be over-simplified. 
The success postulate, for instance, 
has been questioned on the grounds that 
when one receives a new piece of 
information, it is hardly the case that he/she accepts 
it unconditionally  
\cite{Makinson97,HanssonFCF01,Ferme99,conf/kr/BoothFKP12,Ma11b}. 
Others have pointed out that the assumption 
of closure of $X$ by logical consequences should be 
relaxed, for even though two rational agents both commit to believe the same 
belief set, they may primarily believe 
some sub-set of it. They argue that such variance 
should have an effect on the outcome of a belief revision.  
Some others point out that the activity of belief revision should 
be understood more in sequence than as a one-off 
snapshot \cite{Boutilier93,Nayak294,Rott00,Darwiche97,Ma11}, countenancing 
that an epistemic state a rational agent is in should not be  
the same before and after a revision, and that 
the epistemic state should determine how the next revision 
is carried out. 
Still some others 
reflect that a belief can be not only revised but 
also updated \cite{katsuno92}. In case of belief revision, 
agents are discovering about the world they are put in. 
But the world itself may change, in correspondence to which 
beliefs of a rational agent's also alter.  
To sum up, after so much was achieved in \cite{Makinson85}, 
still so much is being investigated around the logical foundation
of belief revision. In this work of ours, we  
consider the relation between beliefs 
and agents' perception 
about them. As Ma et al. 
showed in \cite{Ma11}, in a perception-based model 
we can assume that what agents accept 
is not $P$ but what $P$ appears to him/her. 
We retain the promising feature in our approach. However, we also 
argue and consequently model that 
rational agents are, in reality, accepting what may potentially 
come to be $P$, even if 
only his/her perception of it was originally accepted.
Roughly, 
our belief acceptance model is as follows.  \nolinebreak
\begin{enumerate}[leftmargin=0.4cm]
    \item A belief base consists of 
        two types of beliefs. 
        One comprises all 
        that are reachable by the rational 
        agent through logical consequences. This type forms a belief base in the traditional
        AGM sense. The other comprises 
        all the other beliefs associated to 
        the visible beliefs, which are there, but 
        which are not presently visible to the rational 
        agent. These may be called latent beliefs.\footnote{The interpretation of a latent belief varies. 
In \cite{Altman06}, for instance, it is explained as a belief that 
can be introspectively discovered. In this work, 
such a belief is considered visible.}
    \item There comes an external information $\{P\}^\diamond$. It is 
        a collection of propositions, and 
        corresponds to an input $P$ 
        in the traditional AGM sense. Whatever the content 
        of the belief set is, $P$ out of 
        $\{P\}^\diamond$ is always accepted. 
        $P$ is therefore some essence  
         of the external information 
        $\{P\}^\diamond$. 
        All the other propositions are attributive to 
        $P$. How many of them the rational agent sees 
        depends 
        on what beliefs are visible to him/her in his/her 
        belief set.
    \item Upon acceptance of $\{P\}^\diamond$, therefore, 
        the visible part of $\{P\}^\diamond$ is accepted into 
        the agent's belief set. If any of them contradicts 
        his/her existing visible beliefs, then 
        a necessary belief contraction first takes place. 
        {\it However, 
        the invisible part of $\{P\}^\diamond$ is also 
        accepted as latent beliefs.} 
    \item The visible part of $\{P\}^\diamond$, meanwhile, 
        may stimulate existing latent beliefs 
        in the agent's belief set, making them visible. 
\end{enumerate}
By 1. a belief set in this model is larger than 
one in the traditional model, since it also 
contains latent beliefs. 
By 3. even though 
only the part of an external information visible 
to a rational agent is consciously accepted into 
his/her belief set, the remaining are also unconsciously 
accepted.  We now look into some details of 
our model. \\
\textbf{Justification of the essence of 
    external information} 
In our perception model, the visible part of the belief set 
of a rational agent in a way represents his/her mind. 
The mind determines how an external information appears to him/her.
Suppose now, however, 
that every visible belief in his/her belief set 
had come originally from the external world. Then 
there must have been a moment when he/she had no beliefs 
in his/her belief set. Thence arises a question. 
Given $\{P\}^\diamond$, what of $\{P\}^\diamond$, provided 
that it can be believed, does the empty mind accept, 
and how does he/she do it? Our supposition is that it must 
be accepted almost axiomatically. 
Perhaps an example helps here.  
Suppose a situation where 
we are instructed to have 
$\{P_x\}^\diamond = \ ${\tiny \pmglyph{\Hi-\Hw-\Hplus-\Hm-\HI-\HGi-{\HQiii:\HDxlvi}-\Hibp}}, as our belief. 
It does not appear at all like a belief. But anyway we 
are accepting it as such, for we are instructed to do so. 
Now, acceptance of {\tiny \pmglyph{\Hi-\Hw-\Hplus-\Hm-\HI-\HGi-{\HQiii:\HDxlvi}-\Hibp}}
in this manner  should not reveal so much 
about that which {\tiny \pmglyph{\Hi-\Hw-\Hplus-\Hm-\HI-\HGi-{\HQiii:\HDxlvi}-\Hibp}} 
is. In case the mind is empty, the revelation should be 
minimal. The minimal revelation 
is the essence of external information. $P_x$ 
for $\{P_x\}^\diamond$ is chosen to be it. It is meant
to imitate the Kant's {\it Form}  
\cite{Kant08} for beliefs. \\
\textbf{Justification of latent beliefs} 
If it must be the case that {\tiny \pmglyph{\Hi-\Hw-\Hplus-\Hm-\HI-\HGi-{\HQiii:\HDxlvi}-\Hibp}} be categorically accepted 
as a belief,
we, ourselves as a rational agent, go ahead and do so. Suppose that $P_x$ is `{\tiny \pmglyph{\Hi-\Hw-\Hplus-\Hm-\HI-\HGi-{\HQiii:\HDxlvi}-\Hibp}} is a belief.', and that 
it 
was consciously 
accepted into our belief set $X$. Now, suppose 
that we subsequently search on web (Further Acquisition 
of Beliefs) to identify that {\tiny \pmglyph{\Hi-\Hw-\Hplus-\Hm-\HI-\HGi-{\HQiii:\HDxlvi}-\Hibp}} means `The bird 
    eats.' in English translation. By this time 
our belief set is $X'$, and we have a 
revealed attributive belief to {\tiny \pmglyph{\Hi-\Hw-\Hplus-\Hm-\HI-\HGi-{\HQiii:\HDxlvi}-\Hibp}}, 
namely, `The bird eats.' and also `\ {\tiny \pmglyph{\Hi-\Hw-\Hplus-\Hm-\HI-\HGi-{\HQiii:\HDxlvi}-\Hibp}} means The bird eats.\ '.  
Evidently, they are a part of that which  
{\tiny \pmglyph{\Hi-\Hw-\Hplus-\Hm-\HI-\HGi-{\HQiii:\HDxlvi}-\Hibp}}
is, stimulated to come to the conscious mind by some external information. 
It would be strange to think that the possibility that they would become visible emerged 
all of a sudden at $X'$. The possibility was already there 
at $X$ as a potential. 
This indicates that by accepting an external 
information, which can be but a view of it in fact, 
a rational agent unconsciously also accepts 
the external belief as a larger construct, even though much may remain hidden to him/her for a prolonged period. This 
is the justification of the use of latent beliefs. \\
\textbf{Associations} 
Now, because latent beliefs cannot 
be consciously found, they also cannot 
be found through logical consequences 
of visible beliefs. 
We therefore need to have another kind of links 
that connect beliefs.  
For our purposes, they are called associations. Specifically, 
given three propositions (beliefs) $P_1, P_2$ and 
$P_3$, $P_1$ is said to be associated 
to $P_3$ through $P_2$, provided (1) that 
neither $P_2$ nor $P_3$ is a logical consequence 
of $P_1$, (2) that $P_1$ is a logical consequence 
of neither $P_2$ nor $P_3$, and (3)  that $P_3$ becomes visible in case 
$P_2$ is held visible in the agent's belief set. 
In case $P_2$ is not visible, $P_3$ is a latent belief. 
We shall denote the connection by 
the notation 
$P_1(P_2, P_3)$. $\{P_1\}^\diamond$ then comprises
$P_1$ and certain number of $P_1(P_2, P_3)$ attributive 
to $P_1$. The 
$P_1(P_2, P_3)$ is basically $P_3$    if an agent can see $P_3$.
Therefore it might be properer to call $P_3$ an attributive 
belief instead of $P_1(P_2, P_3)$. But for our 
purposes, we will simply call the latter an attributive belief (to $P_1$). 
An attributive belief $P(P_1, P_2)$, 
unlike `$P_1$ implies $P_2$', or `$P_2$ holds if 
a belief set is revised with $P_1$' \cite{Darwiche97}, 
is inevitably dependent on $P$, and is a part of $\{P\}^\diamond$.
\\
\textbf{Outlines, contributions, 
    related works and other remarks} 
To speak of our approach, we represent belief sets and 
external information both as a set of propositional formulas and
of triples of them, and present corresponding postulates  
to the AGM postulates. 
There are of course more to be desired. Nonetheless, 
it is a reasonable point to begin our modelling with. 
Adaptation to other representations can be 
sought after later. We formalise our ideas in Section 2. 
In Section 3, we present postulates for belief expansion, contraction, 
and revision, and observe an interesting phenomenon 
that even if  visible beliefs and an external information 
are both consistent, the result of belief revision
can lead to an inconsistent belief set. Informally,  
suppose  {\tiny 
        \pmglyph{\Hi-\Hw-\Hplus-\Hm-\HI-\HGi-{\HQiii:\HDxlvi}-\Hibp}} again. It is originally accepted 
into $X$ as $P_x$. 
 Suppose that $X$ has (rather irrationally) 
 `It is not the case that the bird eats.' Then 
 of course acceptance of {\tiny 
        \pmglyph{\Hi-\Hw-\Hplus-\Hm-\HI-\HGi-{\HQiii:\HDxlvi}-\Hibp}} would be contradictory to one of the beliefs 
in $X$. However, it is not contradictory at $X$, since 
to $X$, only a non-contradictory view of it is available.
However, at $X'$, some external information reveals 
a hitherto hidden part  of {\tiny 
        \pmglyph{\Hi-\Hw-\Hplus-\Hm-\HI-\HGi-{\HQiii:\HDxlvi}-\Hibp}}, namely, `The bird eats.' Hence 
at $X'$, it turns out that $X'$ is inconsistent; 
furthermore it turns out that it (should) have been all along 
inconsistent from the moment it 
was accepted. We may say that a potential inconsistency
was materialised at $X'$ 
through revelation of a latent belief attributive to an existing belief. Compared to the inconsistency in a typical 
belief revision theory that is caused by 
a new information contradicting old beliefs, 
the inconsistency that we have just mentioned may be already 
in the old beliefs. \\
\indent As far as can be 
gathered from existing works in this research field, the 
idea that we have gone through, i.e. the incorporation of 
the kind of latent beliefs and of their revelation 
through acquisition of beliefs, 
is new, and in a way models our inspiration which 
is often 
triggered by seemingly unrelated information. 
Typically, agents in belief revision theory 
accept an external information as that which it really is. 
But that cannot be 
generally achievable in our model. Ma et al. \cite{Ma11}, on the other hand, 
explicitly looks at a perception model, albeit 
belief revision not on propositions but on possible-world-based epistemic 
states. However, their model focuses on 
agents' interpretation of uncertain inputs, not on 
the latent beliefs: in their work, 
belief revision is subject to the perceived inputs, 
but the parts that are imperceptible at the time 
of belief revision have no effect on the result 
of a belief revision. 
Also,  as they note, 
external information being represented as epistemic states, 
conjunction of 
two epistemic states are undefinable in their method. It is 
also the case, at least for now, that 
they have dealt only with belief revision, while 
expansion and contraction are left for a future 
work.  We present all the three, and show 
the relations between them 
just as in the AGM theory.  
It could be that our approach 
may turn out to be useful also for addressing 
some of the technical challenges in their work. 
Several 
works \cite{Katsuno91,Benferhat05,Ma12b} consider  
cases where an agent may have several views 
about his/her beliefs. 
But these are concerned with agents' certainty about 
their beliefs, not about latent beliefs. 
We will conclude our work in Section 4 with an example. 
\section{Formalisation}        
Readers are referred to Section 1 for the intuition of the key notations. 
We assume a set of possibly uncountably many 
atomic propositions. We denote the set by 
$\mathcal{P}$, and refer to each element by 
the letter $p$ with or without a subscript. 
More general propositions are 
constructed 
from $\mathcal{P}$ and the logical 
connectives of 
propositional classical logic:  
$\{\top_0, \bot_0, \neg_1, \wedge_2, \vee_2\}$.  
    The subscripts denote the arity.  
    We do not specifically use the classical implication
    $\supset_2$, 
    but it is derivable in the usual manner 
    from $\vee$ and $\neg$, as $p_0 \supset p_1 \equiv 
    \neg p_0 \vee p_1$. 
    We denote the set of literals by \textsf{Lit}. We denote the set of all the propositions 
    by $\props$, and refer to each element 
    by $P$ with or without a subscript. 
    These model beliefs. In the rest, 
    we use the two terms: propositions and beliefs, almost 
    interchangeably. We only prefer to use `proposition' when 
    external information is involved, and `belief' when 
    the proposition is in a belief set. 
    To model the concept of associations 
    among propositions/beliefs, we also define what we call attributive 
    propositions/beliefs.  Let us assume that, given 
    any $O \subseteq 2^{\props}$, $\logcon(O)$ is 
    the set of all the propositions that 
    are the logical consequences of any (pairs of) elements in $O$. 
    We say that a set of propositions/beliefs $U$ is 
    consistent iff for any 
    $P \in \props$, if $P \in U$, then $\neg P \not\in U$, 
    and if $\neg P \in U$, then $P \not\in U$. 
    We assume that, 
    given any pair of some sets $(U_1, U_2)$, 
    $\pi_0((U_1, U_2)) = U_1$ and 
    $\pi_1((U_1, U_2)) = U_2$. 
    We further assume that the union of two pairs of 
    sets: $(U_1, U_2) \cup (U_1', U_2')$ 
    is $(U_1 \cup U_1', U_2 \cup U_2')$. 
    \begin{definition}[Associations and attributive beliefs] 
        We define an association tuple 
        to be a tuple $(\mathcal{I}, X,  \assoc)$ where 
        $\mathcal{I}$ is a mapping 
        from $\lit$ to  $2^{\props \times 
        \props}$; 
        $X$ is an element of $2^{\props}$; and 
        $\assoc$ is a mapping from 
        $\props$ to $2^{\props \times \props}$. 
        Let $\exclude$ be a mapping 
        from $\props$ to $2^{\props}$ such that 
        $\exclude(P) = \logcon(\{P\}) \cup 
        \{P_1 \in \props \ | \ P \in \logcon(\{P_1\})\}$. 
        Then 
        $\mathcal{I}$ is defined to satisfy that, 
       for any $P \in \lit$, if either  
       $P_1 \in \exclude(P)$ or $P_2 \in 
       \exclude(P)$, then $(P_1, P_2) \not\in 
       \mathcal{I}(P)$. 
       $\assoc$ is defined to satisfy (1) that if $P$ is a tautology, 
       then $\assoc(P) = (\emptyset, \emptyset)$; (2) that 
       if $P$ is inconsistent, i.e. $\neg P$ is a
       tautology, then $\assoc(P) = (\props, \props)$; and (3) 
       that, if neither,  
       {\small 
           \begin{enumerate}[leftmargin=0.4cm] 
               \item $\assoc(P) = \mathcal{I}(P)$ if $P \in \textsf{Lit}$.   
               {\normalsize
               {\it Explanation}: The association 
                   links are obtained recursively, and 
                   so it suffices that $\mathcal{I}$ 
                   be defined only for the elements of \textsf{Lit}.  
               }
           \item $\assoc(P_1 \wedge P_2) = 
               (\assoc(P_1) \cup \assoc(P_2)) 
               \downarrow \exclude(P_1 \wedge P_2)$ 
               where $(U_1, U_2) \downarrow U_3
               = (U_1', U_2')$ such that 
               $U_{1,2}' = U_{1,2} \backslash U_3$.   
               {\normalsize 
               {\it Explanation}:  $P_1$ 
               and $P_2$ each has certain set of 
               attributive beliefs/propositions, 
               say $U_{P_1}$ and respectively $U_{P_2}$. 
               Then it is more natural to think 
               that the elements of $U_{P_1} \cup U_{P_2}$ 
               are recognised attributive to 
               $P_1 \wedge P_2$. The only exceptions are 
               those beliefs/propositions which are 
               connected to $P_1 \wedge P_2$ 
               by logical consequences. These are excluded.  
           }
           \item $\assoc(P_1 \vee P_2)
               = \assoc(P_1 \wedge P_2)$ 
               if $P_1, P_2 \in X$. 
               {\normalsize 
               {\it Explanation}: In this paper, 
               that a belief/proposition  
               is in some $X$ means, informally, that 
               it is a [held belief]/[true proposition]
               in $X$, as to be evidenced in 
               the next section. Now, if both $P_1$ and $P_2$ 
               are held/true in $X$, then
               there is only a locutionary 
               difference between 
               $P_1 \vee P_2$ and $P_1 \wedge P_2$. 
           }
           \item $\assoc(P_1 \vee P_2)\! =\! 
               \assoc(P_i)$ if 
               $\neg P_j \! \in\! X$ 
               for $i, j\! \in\! \{1,2\}, i\! \not=\! j$.   
               {\normalsize 
                   {\it Explanation}: If 
                   $\neg P_j$ is in $X$,  
                   then if $X$ is consistent at all, 
                   $P_1 \vee P_2$ is a [held belief]/[true
                   proposition] in $X$ 
                   just because $P_i$ is a [held belief]/[true 
                   proposition] in $X$.
                   \hide{ $P_2$ may or may not be 
                   in $X$. 
                   If it is in $X$, then 
                                      And so $\assoc(P_1 \vee P_2) = \assoc(P_2)$. 
                   If $\neg P_2$ is in $X$, 
                   then $X$ cannot be a consistent 
                   set if it contains $P_1 \vee P_2$, because 
                   it also contains $\neg (P_1 \vee P_2)$. 
                   So we do not - at least in this paper 
                   where the inconsistency tolerance is 
                   not the focus - care what
                   beliefs are attributive to $P_1 \vee P_2$. 
                   If neither $P_2$ nor $\neg P_2$ is in $X$, 
                   then either of the two possibilities 
                   holds. Whichever is the case, it makes sense to 
                   have $\assoc(P_1 \vee P_2) = \assoc(P_2)$ here. 
                   Similarly when $\neg P_2$ is in $X$.  
               }
               }  
           \item $\assoc(P_1 \vee P_2) 
               = \assoc(P_i) \downarrow
               \exclude(P_1 \wedge P_2)$ 
               if $P_i \in X$ and 
               $P_j, \neg P_j \not\in X$ 
               for $i,j\in\{1,2\}$, $i\not=j$. 
               Explanation: If $P_i$ is in $X$, 
               but not the other or negation of 
               the other, then the rational agent 
               holding $X$ may or may not believe
               $P_j$. But any pair $(P_x, P_y)$ 
               which is in the set on the right 
               hand side of the equation 
               is included in either of the two 
               cases, and is safe to assume. 
           \item $\assoc(P_1 \vee P_2)$ 
               consists of all the pairs $(P_x, P_y)$ satisfying
                the following, otherwise: there exists 
                $(P_a, P_{A})$ in  
                $\assoc(P_1)$ and 
                there exists $(P_b, P_{B})$ in 
                $\assoc(P_2)$ such that
                (1) $P_x = P_a$; (2) $\logcon(P_a) = \logcon(P_b)$;
                (3) either $P_{B} \in \logcon(P_{A})$
                or $P_{A} \in \logcon(P_{B})$; 
                (4) if $P_{B} \in \logcon(P_{A})$, 
                then $P_y = P_{B}$, 
                else if $P_{A} \in \logcon(P_{B})$, 
                then $P_y = P_{A}$; and (5) 
                $P_x, P_y \not\in \exclude(P_1 \wedge P_2)$.  
                {\normalsize Explanation:  
                For the other cases, if $X$ 
                is consistent at all, then 
                it is not clear if 
                $\assoc(P_1 \vee P_2)$ is 
                $\assoc(P_1 \wedge P_2)$ or 
                $\assoc(P_i)$, $i \in \{1,2\}$. 
                But if $P_y$ becomes visible 
                under the same condition ($P_x$) in each of the three 
                cases, then we know that $(P_x, P_y)$ is a pair that is safe to assume.  
            } 
           \item $\assoc(\neg (P_1 \wedge P_2)) 
            = \assoc(\neg P_1 \vee \neg P_2)$. 
        \item $\assoc(\neg (P_1 \vee P_2)) = \assoc(\neg P_1 
            \wedge \neg P_2)$. 
       \end{enumerate}     
          } 
Now, assume 
       that $P$ is a proposition/belief which is neither 
       tautological nor inconsistent.  
       Let us call each $P(P_1, P_2)$ for 
       some $P, P_1, P_2 \in \props$ 
       a belief triplet, and let us denote 
       the set of belief triplets by $\btriplet$. 
       Then we define the set 
       {\small $\{P(P_1, P_2) \in \btriplet \ | \ [(P_1, P_2) \in 
           \assoc(P) ] \andMeta \assoc(P) \not= 
           (\props, \props)]\}$}\footnote{In lengthy formal 
    expressions, we use meta-connectives 
    $\andMeta, \orMeta, \rightarrow^{\dagger}, \forall, \exists$ in place 
for conjunction, disjunction and material implication, 
universal quantification and existential quantification, each 
following the semantics in classical logic.  }
 to be 
   the set of propositions/beliefs attributive 
   to $P$. We denote the set by $\cond(P)$. 
   We denote $\bigcup_{P \in \props} \cond(P)$ 
   simply by $\cond$. 

   \end{definition}     
\begin{figure*}[!t]   
\begin{center}  
        \scalebox{0.95}{ 
            \begin{tabularx}{\textwidth}{XX}     
                \textbf{Belief expansion} \\
                $1. \cna(\xa) \plusa \qa = \cna(\cna(\xa) 
                \cup \qa)$.
                \\
                \textbf{Belief contraction}   
                & 
\textbf{Belief revision} \\
$1. \cna(\xa) \diva \qa = \cna(\cna(\xa) \diva \qa)$ 
 (Closure). 
                & 
                $1. \cna(\xa) \stara \qa = \cna(\cna(\xa) \stara
                \qa)$ (Closure). \\
                $2. \qa \not\in \cna(\emptyset)  
                \rightarrow^{\dagger}  
                \qa \not\in \cna(\cna(\xa) \diva \qa)$ 
                (Success). 
                & 
                $2. \qa \in \cna(\xa) \stara \qa$ 
                (Success). \\
                $3. \cna(\xa) \diva \qa \subseteq 
                \cna(\xa)$ (Inclusion). 
                & 
                $3. \cna(\xa) \stara \qa 
                \subseteq \cna(\xa) \plusa \qa$ 
                (Inclusion). 
                \\ 
                $4. \qa \not\in 
                \cna(\xa) \rightarrow^{\dagger} 
                \cna(\xa) \diva \qa = \cna(\xa)$ 
                (Vacuity). 
                & 
                $4. [\neg \qa \not\in \cna(\xa)] \rightarrow^{\dagger}
                    [\cna(\xa) \stara \qa = \cna(\xa) \plusa 
                    \qa]$ (Vacuity). \\
                $5. [\qa_1 \leftrightarrow  
                \qa_2 \in \cna(\emptyset)] 
                \rightarrow^{\dagger} [\cna(\xa) \diva \qa_1 = \cna(\xa) \diva \qa_2]$ 
                (Extensionality). 
                & 
                $5. [\qa_1 \leftrightarrow \qa_2 \in 
                \cna(\emptyset)] \rightarrow^{\dagger} 
                [\cna(\xa) \stara \qa_1 = \cna(\xa) 
                \stara \qa_2]$ (Extensionality). \\ 
                $6. \cna(\xa) \subseteq 
                (\cna(\xa) \diva \qa) \plusa \qa$ (Recovery). 
                & 
                $6. \cna(\xa) \stara \qa$ is consistent 
                if $\qa$ is consistent (Consistency). \\
                $7^*. [\qa_1 \not\in  
                \cna(\xa) \diva (\qa_1 \wedge \qa_2)] 
                \rightarrow^{\dagger} 
                [\cna(\xa) \diva (\qa_1 \wedge \qa_2) 
                \subseteq \cna(\xa) \diva \qa_1]$ 
                (Conjunctive inclusion). 
                & 
                $7^*. \cna(\xa) \stara (\qa_1 \wedge \qa_2) 
                \subseteq (\cna(\xa) \stara \qa_1) 
                \plusa \qa_2$ (Super-expansion). 
                \\
                $8^*.(\cna(\xa) \diva \qa_1) 
                \cap (\cna(\xa) \diva \qa_2) 
                \subseteq \cna(\xa) \diva (\qa_1 \wedge 
                \qa_2)$ (Conjunctive overlap).  
                & 
                $8^*.[\qa_1 \not\in \cna(\xa) \stara 
                \qa_2] \rightarrow^{\dagger} 
                [(\cna(\xa) \stara \qa_2) \plusa \qa_1 
                \subseteq \cna(\xa) \stara (\qa_1 \wedge 
                \qa_2)]$ (Sub-expansion). 
    \end{tabularx}   
}
\end{center}   

\caption{The AGM postulates for belief expansion, contraction 
    and revision. The last two postulates 
    for belief contraction and for belief revision 
    are supplementary postulates that  
    regulate  belief retention. $\qa_1 \leftrightarrow 
    \qa_2$ is an abbreviation 
    of $(\neg \qa_1 \vee \qa_2) \wedge
    (\neg \qa_2 \vee \qa_1)$. 
}  
\label{postulates} 
\end{figure*}

    \hide{ The definition in (the 3rd, the 4th and) the 5th 
    item says that $\assoc$ may be non-deterministic 
    for disjunction, depending on the content of $X$.} 
\hide{ 
    We call any triple of propositions/beliefs 
       $P_1(P_2, P_3)$ for some $P_1, P_2, P_3 \in \props$ 
       a belief triplet. We denote the set of 
       belief triplets by $\btriplet$.  
   \begin{definition}[Attributive propositions/beliefs]  
           Assume some association tuple 
           $T := (\mathcal{I}, X, \assoc)$. Let $P$ be a proposition which is neither 
       tautological nor inconsistent. 
       We define the set {\small $\{P(P_1, P_2) \in \btriplet \ | \ [(P_1, P_2) \in 
           \assoc(P)] \andMeta [\assoc(P) \not= 
           (\props,\props)]\}$}\footnote{In lengthy formal 
    expressions, we use meta-connectives 
    $\andMeta, \orMeta, \rightarrow^{\dagger}, \forall, \exists$ in place 
for conjunction, disjunction and material implication, 
universal quantification and existential quantification, each 
following the semantics in classical logic.  }
      to be the set of propositions/beliefs attributive 
       to $P$ under $T$. We denote the set by 
       $\cond_T(P)$. We denote 
       $\bigcup_{P \in \props}\cond_T(P)$ simply 
       by $\cond_T$.  
   \end{definition} 
}
    \noindent Observe that there is no belief attributive 
    to a tautological or an inconsistent belief. 
    \hide{ Observe also that 
    there may be more than one association tuples 
    for a given $\mathcal{I}$ and a given $X$, depending 
    whether disjunction occurs in subformula(s) of $P$.}\\
    \indent To characterise rational agents, 
    let us denote $2^{\props} \times 2^{\btriplet}$ 
    by $\bbase$, and call each element of 
    $\bbase$ a belief base.  We denote any element 
    of $\bbase$ by $B$ with or without a subscript. 
    \hide{        We say that 
    a belief base $B := (U, V)$ is consistent 
    under some $\mathcal{I}$ and 
    some $\assoc$ 
    iff (1) $U$ is consistent and 
    (2) $V = \cond_{(\mathcal{I}, U, \assoc)}$.   
}
              Instead of working directly on a belief base, however, 
    just as in the AGM theory we represent a rational agent 
    as some belief set.  
    \hide{ 
Let us denote 
{\small $\bigcup_{P \in \props}\{\cond_T(P) \ | \  
        T = (\mathcal{I}, \logcon(\emptyset), \assoc)  
\}$}
simply by $\cond$. It follows straightforwardly 
that {\small $\bigcup_{P \in \props}\{\cond_{T'}(P) \ | \ T' = 
        (\mathcal{I}, \logcon(X), \assoc) \text{ for 
            some consistent } X\} \subseteq \cond$}.  
We define a belief base 
to be 
a pair in $2^\props \times 2^{\cond}$. Here, 
we assume that an element of $\cond$ 
is $\bigcup_{P \in \props} \{M \in \cond_{(\mathcal{I}, 
        \logcon(\emptyset), \assoc)}(P)\}$. } 
\begin{definition}[Belief sets] 
    \label{belief_set} 
    Let $Cn$ be a closure operator on 
    $\bbase$  
    such that  
    $Cn(B)$ is 
    the least fixpoint  
    of $Cn^k(B)$ ($k \ge 0$), defined by:  
    \begin{itemize}[leftmargin=0.3cm] 
        \item {\small $\pi_0(Cn^0(B)) = 
                \logcon(\pi_0(B))$}.  
            \hide{\item {\small $T^k = 
                (\mathcal{I}, \pi_0(Cn^k(B)), \assoc_k)$} 
            if {\small $\pi_0(Cn^k(B))$} is consistent. 
        }
\item {\small $\pi_1(Cn^k(B)) = 
                \bigcup_{P \in 
                    \pi_0(Cn^k(B))} \cond(P)$} 
                \item {\small $\pi_0(Cn^{k+1}(B)) = 
                \logcon(\pi_0(Cn^k(B)) \cup 
                \mathsf{A^{k}})$}.\\
                \end{itemize}   
where  
            {\small $\mathsf{A^k} = 
            \{P_2 \ | \ [P(P_1, P_2) \in \pi_1(Cn^k(B))]
                \andMeta [P_1 \in \pi_0(Cn^k(B))]\}$}.
        $Cn(B)$ is assumed compact. 
 We say that a belief base $B$ is closed iff 
 $B = Cn(B)$. 
  We call any closed belief base a belief set. 
\end{definition}    
\noindent  
{\it Informal explanation}: $B$ and $Cn^i(B)$ have 
two components, the first of which contains 
propositions, and the second of which 
contains belief triplets. 
The starting point for the fixpoint iteration is 
$Cn^0(B)$. The first 
component of $Cn^0(B)$, i.e. $\pi_0(Cn^0(B))$, 
contains all the propositions as result from 
taking the $\logcon$ closure on $\pi_0(B)$.  
The second component of $Cn^0(B)$, i.e. 
$\pi_1(Cn^0(B))$, contains all the 
associated 
attributive propositions 
to each proposition in $\pi_0(Cn^0(B))$. 
 In the next round of $Cn$ application
to $Cn^0(B)$ to obtain $Cn^1(B)$, we first obtain 
a set of propositions, $\mathsf{A^0}$. 
If $P \in \mathsf{A^0}$, then (1) $P$ occurs in 
$Cn^0(B)$ as 
$P_x(P_y, P)$ for some $P_x$ and some $P_y$, 
and (2) the $P_y$ occurs in $Cn^0(B)$. 
The $\mathsf{A^0}$ is added to the first component of $Cn^0(B)$, 
which is then closed under logical consequences. 
This set is the first component of $Cn^1(B)$. 
The second component of $Cn^1(B)$ then contains  
all the attributive beliefs associated to them. This process 
is repeated, eventually reaching the fixpoint, 
$Cn(B)$.  
\hide{ To speak of the $M \in \cond_T(P)$ and the requirement 
that $\pi_1(Cn^i(B)) \subseteq M$ for $i \ge 0$,
the intentions behind stating the condition are as follows: (1) 
firstly, if $\cond_T(P)$ must be constructed in an incremental 
manner 
along this fixpoint iteration, we would not be able to tell, 
in advance of the iteration, 
what beliefs are attributive to each proposition. Therefore
it must be given that 
$T := (\mathcal{I}, X, \assoc)$ is known in advance. Note that 
this 
assumption does not lead to a paradox. Similar assumptions 
are taken for the knowability of truth within classical logic. (2) Then,
knowing $\cond_T(P)$, the fixpoint iteration 
should assume, for each proposition, those attributive propositions 
corresponding to it in some $M \in \cond_T(P)$. (3) Now, 
clearly, $M$ should not be swapped to another $M' \in \cond_T(P)$ 
during the fixpoint iteration. $\pi_1(Cn^i(B)) \subseteq M$ 
for $i \ge 0$ enforces that the requirement be respected.} $\Box$\\
\begin{observation}[Adequacy] 
    For any belief set $Cn(B)$ and for any $P \in \props$, 
 if $P \not\in Cn(B)$, then $Cn(B)$ does not 
 contain any beliefs attributive to $P$.\footnote{Adequacy 
     may better be a postulate on its own 
     than 
     be integrated into $Cn$. 
     We choose this way here only because 
     it keeps the number of postulates 
     on par with that in the AGM theory.}  
\end{observation}  
A belief set $Cn(B)$ is the representation of 
what a rational agent holding it is committed to believe.  
\hide{
Meanwhile, just for reference, what a putative irrational agent 
who sees everything, even beyond what he/she 
can perceive, 
would be committed to believe, 
which we denote  
by $Cn^{\bullet}(B)$, is 
defined in a similar way to $Cn(B)$, 
but
instead of $\mathsf{A}^k$, 
$\mathsf{A}^{\bullet k}= \{P_2 \ | \ P(P_1, P_2) 
    \in \pi_1(Cn^{\bullet k}(B))\}$ is used. 
The two sets do not generally coincide. Even 
when the same $T^k$s, $k \ge 0$, 
are shared, we generally have 
$Cn(B) \subseteq Cn^{\bullet}(B)$. \\ 
}
For the external information,  
we assume that every piece of such information 
is a set of propositions with one primary proposition 
$P$ and other (zero 
or more) propositions attributive to $P$, i.e.
$\{P\}^\diamond = (\{P\}, \cond(P))$. 
The receiving agent may or may not notice 
of any of the attributive propositions  
in $\{P\}^\diamond$ at the time he/she accepts $\{P\}^\diamond$. 
That depends on $\pi_0(Cn(B))$. 
Let 
us denote by $\visible_B(\{P\}^\diamond)$ 
the part of $\{P\}^\diamond$ visible to $Cn(B)$. 
Specifically, 
$\visible_B(\{P\}^\diamond)$ is defined to be 
{\small $P \cup \{P_2 \ | [P(P_1, P_2) \in 
        \pi_1(\{P\}^\diamond)] \andMeta 
        [P_1 \in \pi_0(Cn(B))]\}$}.
\section{Postulates and Representations}   
We first of all state one postulate about the relation 
between an association tuple and a belief set. 
\begin{enumerate}[leftmargin=0.5cm]
    \item For any belief set $Cn(B)$, 
        $(\mathcal{I}, \pi_0(Cn(B)), \assoc)$ is 
        the association tuple for it. 
      \end{enumerate}  
We have two principles for expansion. 
The AGM postulates for expansion, contraction and 
revision are listed in Figure \ref{postulates} for 
easy comparisons. 
$\cna$ is a closure operator by logical consequences;
$\xa$ is a set of propositions; 
$\qa$ is a proposition as an external information, i.e. all in the AGM sense \cite{Makinson85}. 
\begin{enumerate}[leftmargin=0.5cm] 
    \item $Cn(B) + \{P\}^{\diamond} = Cn(Cn(B) \cup 
        \{P\}^\diamond)$ (Augmentation). 
\item If the  association 
            tuple for $Cn(B)$ is 
            {\small $(\mathcal{I}, \pi_0(Cn(B)), \assoc)$}, 
            then that to $Cn(B) + \{P\}^\diamond$ 
            is $(\mathcal{I}, \pi_0(Cn(B) + \{P\}^\diamond), 
            \assoc)$ (Association update). 
\end{enumerate}
The postulates for contraction are as follows.  
\begin{enumerate}[leftmargin=0.5cm]
    \item $Cn(B) \div \{P\}^\diamond = Cn(Cn(B) 
        \div \{P\}^\diamond)$  
        (Closure).  
    \item {\small $\forall 
        P_1 \in \visible_{B}(\{P\}^\diamond).P_1 
        \not\in \pi_0(Cn(\emptyset))
        \rightarrow^{\dagger} 
        P_1 \not\in \pi_0(Cn(B) \div \{P\}^\diamond 
        )$} (Success).    
    {\it Explanation}: If a proposition 
            that is visible to $Cn(B)$ 
            is not a tautology, then 
            it is not in 
            the contracted set. 
                \item $Cn(B) \div \{P\}^\diamond \subseteq 
            Cn(B)$ (Inclusion).  
            {\it Explanation}: The belief set 
            to result 
            through belief contraction is a sub-set 
            of the initial belief set. 
        \item {\footnotesize $(\forall P_a \in \visible_B(\{P\}^\diamond). 
                P_a \not\in \pi_0(Cn(B)) \orMeta 
                P_a \in \pi_0(Cn(\emptyset))) \rightarrow^{\dagger}
            Cn(B) \div \{P\}^\diamond = Cn(B)$} (Vacuity). 
            {\it Explanation}: If 
            all the visible propositions of 
            an external information are either 
            a tautology or a proposition 
            not 
            in $Cn(B)$, then the belief contraction
            does not modify 
            $Cn(B)$. 
        \item {\small $[\logcon(\visible_{B}(\{P_1\}^\diamond)) 
            = 
            \logcon(\visible_{B}(\{P_2\}^\diamond))]$}
            {\small $\rightarrow^{\dagger}
            [Cn(B) \div \{P_1\}^\diamond = 
            Cn(B) \div \{P_2\}^\diamond]$} 
            (Extensionality).  
            {\it Explanation}: If 
            $\visible_{B}(\{P_1\}^\diamond)$ and 
            $\visible_{B}(\{P_2\}^\diamond)$ 
            are indistinguishable in content, 
            then contracting 
            $Cn(B)$ by $\{P_1\}^\diamond$ or by 
            $\{P_2\}^\diamond$ is the same. 
                    \hide{
\item {\small $[P_0(P_1, P_2) \in \pi_1(Cn(B) \div \{P\}^\diamond)] \rightarrow^{\dagger}$}\\ 
                {\small $[P_0 \in \pi_0(Cn(B) \div \{P\}^\diamond)]$}
            (Adequacy). 
            {\it Explanation}: Belief contraction 
            ensures that the resultant belief 
            set be adequate. 
        }
            \hide{ \item {\small $[P_x \in \pi_0(Cn(B) \div \{P\}^\diamond)]
                \rightarrow^{\dagger} 
                [\forall P_x(P_y, P_z) 
                \in \pi_1(Cn(B)) .P_x(P_y, P_z) \in 
                \pi_1(Cn(B)\div \{P\}^\diamond)]$} 
            (Conservation). 
            {\it Explanation}: For any 
            belief in the contracted 
            belief set, any belief attributive 
            to the belief found in the original 
            belief set is also found in the contracted  
            belief set.   
        } 
\item {\small $Cn(B)
        \subseteq (Cn(B) \div \{P\}^\diamond)
                + (\visible_B(\{P\}^\diamond), \emptyset)$} 
                (Recovery).  
            {\it Explanation}: 
            If the contracted belief set is 
            expanded with the beliefs that have been removed, 
            then the belief set contains all the 
             beliefs in the original belief set.
        \item If the association 
            tuple for $Cn(B)$ is 
            $(\mathcal{I}, \pi_0(Cn(B)), \assoc)$, 
            then that for $Cn(B) \div \{P\}^\diamond$ 
            is $(\mathcal{I}, \pi_0(Cn(B) \div \{P\}^\diamond), 
            \assoc)$ (Association update).  
\end{enumerate}        
\noindent There is no gratuitous recovery, however. 
Compare it to Recovery in Figure \ref{postulates}.
\begin{observation}
    There is no guarantee that 
    we have  $Cn(B) \subseteq (Cn(B) \div \{P\}^{\diamond})
        + \{P\}^\diamond$. 
\end{observation}    
\hide{ 
\begin{proof} 
    Suppose $B := (\{p_0, p_1, p_2\}, \{p_0(p_1, p_2), 
        p_0(p_2, p_1)\})$ so that 
    $\cond(p_0) = \{p_0(p_1, p_2), p_0(p_2, p_1)\}$, 
    that $\cond(p_1) = \cond(p_2) = \emptyset$ and that 
    none of $\{p_0, p_1, p_2\}$ 
    are a logical consequence of others, or a tautology. 
\end{proof}     
}
\noindent The reason 
is; if a belief set and a contracted belief set of its 
are not identical, then the association tuple associated to 
each of them can be different. \\
\indent In line with the AGM representation of 
the postulates, we define a mapping $\bigtriangleup$ 
from belief sets and propositions 
into belief sets. For any belief set $Cn(B)$ 
and any incoming information $\{P\}^\diamond$, 
we say that $Cn(B_1)$ satisfying
$\pi_0(Cn(B_1)) \subseteq \pi_0(Cn(B))$ is a maximal 
subset of $Cn(B)$ for 
 $\{P\}^\diamond$ iff 
\begin{enumerate} 
    \item  
For any $P_1 \in \visible_B(\{P\}^\diamond)$, 
$P_1 \not\in \pi_0(Cn(B_1))$ if $P_1$ is not a tautology.  
\item  For any $Cn(B_2)$, if {\small $\pi_0(Cn(B_1)) \subset \pi_0(Cn(B_2)) \subseteq 
        \pi_0(Cn(B))$}, then 
there exists some $P_a \in \visible_B(\{P\}^\diamond)$ 
such that $P_a \in \pi_0(Cn(B_2))$.  
\item {\small $Cn(B)= 
        Cn(B_1 \cup (\visible_B(\{P\}^\diamond), 
        \emptyset))$}.\footnote{This condition 
        is derivable from the first two in the AGM theory; 
        but not here, which is therefore explicitly 
        stated.}
\end{enumerate} 
We define 
$\bigtriangleup(Cn(B), \{P\}^\diamond)$ to be 
the set of all the subsets of $Cn(B)$ 
maximal for $\{P\}^\diamond$. We further define 
a function $\gamma$, so that, 
if  $\bigtriangleup(Cn(B), \{P\}^\diamond)$ is not empty, 
then 
$\gamma(\bigtriangleup(Cn(B), \{P\}^\diamond))$ 
is a sub-set of $\bigtriangleup(Cn(B), \{P\}^\diamond)$; 
or if it is empty, it is simply $Cn(B)$. 
\begin{theorem}[Representation theorem 
    of belief contraction]   
    \label{representation} 
    Let $Cn(B)$ be a belief set. 
    Then, 
    it follows 
    that $Cn(B) \div \{P\}^\diamond = 
    \bigcap(\gamma(\bigtriangleup(Cn(B), \{P\}^\diamond)))$. 
\end{theorem}   
\begin{proof}  
    Note that the postulates modulo association tuple 
    can be reduced down to 
    package contraction \cite{Fuhrmann94} 
    of $\visible_B(\{P\}^\diamond)$. 
    However, we show cases of one direction 
    of the proof with details for not very
    straightforward ones. 
    We show that for any 
    particular belief set as results from 
    {\small $Cn(B) \div \{P\}^\diamond$}, 
    there exists some particular $\gamma$ 
    such that 
    {\small $Cn(B) \div \{P\}^\diamond =  
        \bigcap(\gamma(\bigtriangleup(Cn(B), \{P\}^\diamond)))$}. 
    Suppose, by 
    way of showing contradiction, that 
    there exists some $\alpha$ which 
    is either some belief $P_a$  or some 
    attributive belief $P_a(P_b, P_c)$ such that 
    $\alpha \in Cn(B) \div \{P\}^\diamond$ 
    and that $\alpha \not\in \pi_i(\bigcap(\gamma(\Delta(Cn(B), 
    \{P\}^\diamond))))$, where $i$ is 0 or 1, depending 
    on which $\alpha$ is. 
    Suppose $\alpha = P_a$. 
    By Closure and Inclusion, we have that $\alpha \in 
    \pi_0(Cn(B))$. Now, we consider two cases: $\alpha$ 
    is a tautology, or otherwise. 
    In the latter case, there are two possibilities: (1) 
    For all  $P_y \in \visible_B(\{P\}^\diamond)$, we have either 
    that $P_y$ is a tautology or that $P_y \not\in \pi_0(Cn(B))$;
    (2) There is $P_y \in \visible_B(\{P\}^\diamond)$ such that 
    $P_y \in Cn(B)$ and that it is not a tautology.  
For the second case, define $\conflict$ 
    to be {\footnotesize $\bigcup_{\footnotesize \textsf{for all such } P_y}(\{P_x
        \in Cn(B) 
        \ | \ P_y \in \logcon(P_x)\} 
    \cup \{P_u \in Cn(B) \ | \ 
        [\logcon(P_u) = \logcon(P_x \vee P_w)] \andMeta [P_y \in \logcon(P_x)] \andMeta
        [\exists \neg P_z \in Cn(B).P_z \in \logcon(P_w)]\})$}. 
    Then by Success,  $\alpha \not\in \conflict$. 
There are two sub-cases here.  
    If $\alpha \in \logcon(P_y)$ such that 
    $\logcon(P_y) \not= \logcon(\alpha)$ for some such 
    $P_y$, then by the first condition of maximality, 
    we can choose $\gamma$ appropriately so that 
    any selected maximal sub-set(s) include $\alpha$; 
    contradiction to the supposition. 
    Otherwise, we have that $\alpha \in 
    (Cn(B) \backslash \conflict) \backslash 
    \{P_e \in \logcon(P_y) \ | \ [P_y \in 
        \visible_B(\{P\}^\diamond)]\}$. 
    Contradiction is 
    by the third condition of maximality. 
    Similarly when 
    $\alpha = P_a(P_b, P_c)$. 
\end{proof}  
\indent We finally list postulates for belief revision. 
When one is revising a belief set, he/she is
conscious of the propositions within $\{P\}^\diamond$ that are visible 
to him/her. If his/her belief set $Cn(B)$ 
should contain any contradicting beliefs to them, 
they should be removed. Hence, revision of $Cn(B)$ 
by $\{P\}^\diamond$, which we denote by $Cn(B) * \{P\}^\diamond$, 
perform belief contraction by $\visible_B^-(\{P\}^\diamond)$  
on $Cn(B)$, where $\visible_B^-(\{P\}^\diamond)$ 
is defined to be $\{\neg P_a \ | \ P_a \in 
    \visible_B(\{P\}^\diamond)\}$. 
However, then not $\{P\}^\diamond$ but $(\visible_B(\{P\}^\diamond), 
\bigcup_{\footnotesize P_x \in \visible_B(\{P\}^\diamond)} 
\cond(P_x)))$ expands the contracted belief set. 
\begin{enumerate}[leftmargin=0.3cm]
            \item $Cn(B) * \{P\}^\diamond = Cn(Cn(B) * \{P\}^\diamond)$ (Closure).  
    \item $\visible_B(\{P\}^\diamond) \subseteq \pi_0(Cn(B) * \{P\}^\diamond)$
        (Success 1). 
        {\it Explanation}: Any visible proposition 
        in external information is accepted.  
        \hide{
    \item {\small $
            \bigcup_{\footnotesize P_z \in 
                \visible_B(\{P\}^\diamond)}
                \cond(P_z) \subseteq  
            \pi_1(Cn(B) * \{P\}^\diamond)$} (Success 2). 
        {\it Explanation}: Any attributive
        propositions to any visible 
        proposition in external information are accepted.    
    }
    \item {\small $\forall P_x \in 
            \visible^-_B(\{P\}^\diamond).[P_x \not\in 
            Cn(\emptyset)] \rightarrow^{\dagger}
            [P_x \not\in \pi_0(Cn(B) * \{P\}^\diamond)]$} 
        (Success 2). 
        {\it Explanation}: Any belief 
        contradicting $\visible_B(\{P\}^\diamond)$ 
        is not in the revised belief set. 
        \hide{
\item 
                {\small $\forall \hat{P}(P_x, P_y) \in 
                    \{P\}^\diamond.
                    [P_y \in \visible_B(\{\neg P\}^\diamond)] 
                    \andMeta [P_y \in Cn(B)] 
                    \rightarrow^{\dagger} 
                \hat{P}(P_x, P_y) \not\in Cn(B) * \{P\}^\diamond$} 
            (Rejection).  
            {\it Explanation}: 
            Introduction of attributive 
            beliefs in the form $\hat{P}(P_x, P_y)$ in 
            $\{P\}^\diamond$\footnote{Cf. Observation 
            \ref{ob_adequacy}; as far as attributive
            propositions are concerned, it makes 
            no difference if we consider $\{P\}^\diamond$ 
            or $\{\neg P\}^\diamond$.}
            that are visible to $Cn(B)$ and that are in conflict with 
            any belief(s) in $Cn(B)$ 
            is guarded against.  
        }
        \hide{
    \item {\small $[P_0(P_1, P_2) \in Cn(B) * \{P\}^\diamond] \rightarrow^{\dagger} 
                [P_0 \in Cn(B) * \{P\}^\diamond]$}
            (Adequacy). 
            {\it Explanation}: Belief revision 
            ensures that the resultant belief 
            set be adequate.  
        }
    \item $Cn(B) * \{P\}^\diamond \subseteq Cn(B) + \{P\}^\diamond$ (Inclusion). 
                {\it Explanation}: The belief set 
                that results from revising 
                a belief set with some external information 
                forms a sub-set of the belief set 
                that results from expanding it 
                with the same information. 
            \item {\small $[\forall 
                    P_x \in
                    \visible^-_{B}(\{P\}^\diamond).
                    P_x \not\in \pi_0(Cn(B))] 
                \rightarrow^{\dagger}  
                [Cn(B) * \{P\}^\diamond = 
                Cn(B) + \{P\}^\diamond]$} (Vacuity).   
            {\it Explanation}: If no visible propositions of 
                an external information 
                are in conflict with any beliefs in $Cn(B)$, 
                then the result of 
                revising $Cn(B)$ with it 
                is the same as simply expanding  
                $Cn(B)$ with it.   
            \item {\small $[Cn(\{P_1\}^\diamond) =
                Cn(\{P_2\}^\diamond)]
                 \rightarrow^{\dagger} 
                [Cn(B) * \{P_1\}^\diamond = 
                Cn(B) * \{P_2\}^\diamond]$}
                (Extensionality).  
                {\it Explanation}: If one external information 
                is identical with  
                another in content under $Cn$, then 
                revising $Cn(B)$ with either 
                of them leads to the same belief set. 
                Note the difference from 
                Extensionality for $\div$. Here 
                latent beliefs also matter.  
\item If the  association 
            tuple for $Cn(B)$ is 
            $(\mathcal{I}, \pi_0(Cn(B)), \assoc)$, 
            then that to $Cn(B) * \{P\}^\diamond$ 
            is $(\mathcal{I}, \pi_0(Cn(B) * \{P\}^\diamond), 
            \assoc)$ (Association update). 
        \end{enumerate}
An analogue of Levi identity \cite{Levi77} does not hold here.  
\begin{observation} 
    \label{ob} 
    It is not the case that $Cn(B) * \{P\}^\diamond =  
    (Cn(B) \div \{\neg P\}^\diamond) 
    + \{P\}^\diamond$. 
\end{observation}    
\noindent Also note the following observation. 
\begin{observation}[No AGM consistency upon revision] 
    Even if a belief set $Cn(B)$ and all the elements of $\visible_B(\{P\}^\diamond)$ are consistent, $Cn(B) * \{P\}^\diamond$ may 
    be inconsistent.  
\end{observation}  
\begin{proof} 
  Suppose that we have the following belief set: 
  {\small $Cn(\{p_0, \neg p_2\}, \{p_0(p_1, p_2)\})$} 
  and that $\{P\}^\diamond = (\{p_1\}, \emptyset)$. 
  Suppose that 
  $\cond(p_0) = \{p_0(p_1, p_2)\}$, that 
  $\cond(p_1) = \cond(p_2) = \emptyset$, and that 
  none of $p_0, \ldots, p_2$ 
  are a logical consequence of any others. 
\end{proof}  
\noindent 
Hence, we explicitly have the postulate 
Success 2 for the removal of beliefs in $Cn(B)$, 
instead of the implicit removal through the AGM Consistency postulate.  
Although out of the scope of this work, resolution of this kind of inconsistencies that are 
triggered by 
revelation of latent parts of existing beliefs 
is an interesting problem. 
Despite Observation \ref{ob}, we have the following 
identity theorem. 
\begin{theorem}[Identity] 
    \label{identity}   
    It holds that \\
    {\small $Cn(B) * \{P\}^\diamond = 
        (Cn(B) \div (\visible_B^-(\{P\}^\diamond), \emptyset)) + 
        (\visible_B(\{P\}^\diamond), 
        \bigcup_{\footnotesize 
            P_x \in \visible_B(\{P\}^\diamond)}\cond(P_x))$}.  
\end{theorem} 
\begin{proof}   
We show cases of one direction. 
Details are left to readers. 
Let $Y$ denote
{\small $(Cn(B) \div (\visible_B^-(\{P\}^\diamond), 
\emptyset)) + 
(\visible(\{P\}^\diamond), \bigcup_{\footnotesize 
    P_x \in \visible_B(\{P\}^\diamond)}\cond(P_x))$} for space.  
    Show that for any outcome
    of $Y$, there is 
    an outcome of $Cn(B) * \{P\}^\diamond$ such that 
    $Y = Cn(B) * \{P\}^\diamond$.  
    To show contradiction, suppose 
    some $\alpha$ which is either a 
    belief $P_a$ or an attributive belief 
    $P_a(P_b, P_c)$. Suppose that 
    $\alpha \in \pi_i(Y)$ and 
    that $\alpha \not\in Cn(B) * \{P\}^\diamond$.  
    Three cases here: $\alpha$ is a tautology, 
    $\alpha \in \pi_i(\visible_B(\{P\}^\diamond), \bigcup_{\footnotesize P_z \in \visible_B(\{P\}^\diamond)}
    \cond(\{P\}^\diamond))$, or otherwise. 
    In the last case, there are sub-cases: 
    either 
    $\forall P_y \in \visible^-_B(\{P\}^\diamond).
    P_y \not\in \pi_0(Cn(B))$; or otherwise. 
    For the second case, 
    define $\conflict$ 
    to be {\small $\bigcup_{\footnotesize \textsf{for all such } P_y}(\{P_x
        \in Cn(B) 
        \ | \ P_y \in \logcon(P_x)\} 
    \cup \{P_u \in Cn(B) \ | \ 
        [\logcon(P_u) = \logcon(P_x \vee P_w)] \andMeta [P_y \in \logcon(P_x)] \andMeta
        [\exists \neg P_z \in Cn(B).P_z \in \logcon(P_w)]\})$.} 
    By Closure and Success 
    of $\div$, we must have that 
    $\alpha \not\in \conflict$.  
    There are two sub-sub-cases: (A) 
    $\alpha \in  \pi_i(Cn(B) \div (\visible^-_B(\{P\}^\diamond), \emptyset))$; (B) otherwise.  
For the latter, we have, in case 
    $\alpha = P_a$ (similar 
    if it is $P_a(P_b, P_c)$), that $\bigvee^{\dagger}_{\footnotesize P_z \in \conflict} ([\alpha \in \logcon(P_z)] \andMeta 
    [\logcon(\alpha) \not= \logcon(P_z)])$. 
    That is, $\logcon(\alpha) = \logcon(P_z  \vee P_u) \not= 
    \logcon(P_z)$ for 
    some $P_u$. 
    But by Success 1 and Closure of $*$, $\neg P_z$ is in 
    $Cn(B) * \{P\}^\diamond$.  
    Hence by Closure of $*$, $P_u$ and also  
    $P_z \vee P_u$ are  in $Cn(B) * \{P\}^\diamond$. 
    Contradiction.  
\end{proof}    
\noindent By Theorem \ref{representation} and Theorem \ref{identity}, 
we also gain the representation theorem for belief revision.     
We may add the following supplementary postulates. These 
are all adaptations of the AGM supplementary
postulates (Figure 
\ref{postulates}). 
For any $P_1, P_2$, $Cn(B)$ and 
$\visible_B(\{P\}^\diamond)$,  
let us assume for space that\linebreak ${<B,P,P_1 \wedge P_2, P_1>}$
is almost the same as $\visible_B(\{P\}^\diamond)$ 
except that all the occurrences of $P_1 \wedge P_2 \in \visible_B(\{P\}^\diamond)$ 
are replaced with $P_1$. \\
(For belief contraction)
\begin{enumerate}[leftmargin=0.46cm]
    \item[\textsf{A}] {\small $\forall P_1 \wedge P_2 
            \in \visible_B(\{P\}^\diamond).  
        [P_1 \not\in Cn(B) \div 
        \{P\}^\diamond] \rightarrow^{\dagger} 
        [Cn(B) \div \{P\}^\diamond\\ \subseteq 
        Cn(B) \div <B,P,P_1 \wedge P_2, P_1>
        $} (Conjunctive inclusion).  
\item[\textsf{B}] {\small $
        \forall P_1 \wedge P_2 \in \visible_B(\{P\}^\diamond).
        (Cn(B) \div {<B,P,P_1 \wedge P_2, P_1>
            )} \cap 
        (Cn(B) \div <B,P, P_1 \wedge P_2, P_2>
        ) \subseteq 
        Cn(B) \div \{P\}^\diamond$}. 
    (Conjunctive overlap). 
\end{enumerate} 
(For belief revision)
\begin{enumerate}[leftmargin=0.46cm]
    \item[\textsf{A}] {\small $\forall 
            P_1 \wedge P_2 \in \visible_B(\{P\}^\diamond).
            Cn(B) * \{P\}^\diamond \subseteq  
            (Cn(B) * <B, P, P_1 \wedge P_2, P_1>)
            + (P_2, \cond(P_2))
            $} (Super-expansion).  
    \item[\textsf{B}] {\small $\forall P_1 \wedge P_2 
            \in \visible_B(\{P\}^\diamond).$}\\
        {\small $[{P_1 \not\in 
                {Cn(B)} * <B, P, P_1 \wedge P_2, P_2>}
                ] 
            \rightarrow^{\dagger} [
            (Cn(B) * {<B,P,P_1 \wedge P_2, P_2>
                }) + 
            (P_2, \cond(P_2)) 
            \subseteq 
            Cn(B) * \{P\}^\diamond]$} 
        (Sub-expansion). 
\end{enumerate}  
As in \cite{Makinson85}, addition of these 
postulates ensures that $\gamma$ selects 
the best elements under some criteria. 
Specifically 
{\small $\gamma(\Delta(Cn(B), \{P\}^\diamond)) 
    = \{Cn(B_a) \in \Delta(Cn(B), \{P\}^\diamond) \ | \  
    \forall Cn(B_b) \in \Delta(Cn(B), \{P\}^\diamond). 
    Cn(B_b) \preceq Cn(B_a)\}$} 
    for a total pre-order $\preceq$ over 
    $\Delta(Cn(B), \{P\}^\diamond)$.  
\section{Conclusion with a concluding example}      
\begin{wrapfigure}{r}{0.13\textwidth} 
         \begin{center}
             \resizebox{\linewidth}{!}{
     \begin{tikzpicture} 
           \draw[black,rotate=180,l-system={rule set={X->X+YF,Y->FX-Y},
    step=2pt,angle=90,axiom=FX,order=13}] l-system;    
\node (A) [fill=gray] {};
\end{tikzpicture}    
\begin{tikzpicture}[node distance=0.5cm,
                    semithick,scale=1.5]
  \tikzstyle{every state}=[shape=circle,fill=red,draw=none,text=white]
  \node (A) [fill=red] { };
  \node (B) [below of=A,fill=blue] { };
  \node (C) [below of=B,fill=green]{ }; 
\end{tikzpicture}   
}  
\end{center} 
 \label{fig:subfig1}
 \end{wrapfigure}  
Let us conclude this work with an example. Suppose a scene in an imaginary game application. 
 A player can visit three towns: Town A, Town B and Town C. 
In Town A, he/she obtains a box. 
 On one face of it, 
 there are red, blue and green buttons, as well as 
 a dark pattern. There is 
 a key hole in the pattern, as indicated in grey 
 in the right figure.  But at first the spot is black, 
 just as the rest of the pattern is. 
 Now, if he/she presses the buttons in certain order - 
 red, blue, red, and then green - the inner 
 machinery of the box will actuate and the location 
 of the key hole is illuminated. Whether or not 
 the key hole is being illuminated, if 
 a player has a matching 
 key, it can be inserted into the key hole, and the 
 box will open. Inside the box is a rare item. 
 In Town B, a player obtains the key. In Town C, 
 he/she obtains an enigmatic note: `red, blue, red, green.'
 Now, all of these are the details known to the game developers, 
 but not to the first-time game players.  
Consider the following propositions.  
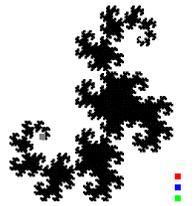
\begin{multicols}{2}  
    {\small 
  \begin{enumerate}[leftmargin=0.4cm]
     \item $p_1$: I have a key.  
     \item $p_2$: There is a key hole. 
     \item $p_3$: I apply the key to the key hole. 
     \item $p_4$: I have a box. 
     \item $p_5$: There are three buttons and a dragon curve. 
     \item $p_6$: There are red, blue and green buttons.  
      \item $p_7$: I have a note having 4 words on. 
      \item $p_8$: There are 4 words: red, blue, red, green, stated in 
          this order.  
      \item $p_9$: Pressing buttons: red, blue, red, and green 
          in this order opens a box.  
      \item $p_{10}$: There is a rare item. 
     \end{enumerate}    
 }
 \end{multicols}
 Suppose that a player 
 visits 
 Town B, then Town A, and then Town C. 
 {\it The first problem (concerning what-is-perceived-is-all-that-there-is)}: 
 Suppose that 
   he/she conceives 
   the beliefs, $p_1, p_2 \supset p_3$ (in Town B), $p_4, p_5$
   (in Town A) and $p_7$ (in Town C) in this order.  
   These are his/her perception of the items. 
   But if we suppose 
   that what is not perceived at the time 
   of perception will have no effect,
   then the player, in Town A, only noted 
   that there were three buttons and a dragon curve 
   on the box. 
   Now, because the critical information, the three 
   colours of the buttons, was not recorded, he/she, upon 
   seeing the note in Town C, did not notice 
   the significance of the 4 words. 
   Consequently, he/she conceived 
   only $p_7$ and not $p_8$. 
   But, then, there is no proper continuity 
   from $p_5 \wedge p_7$ to
   $p_9$ which he/she does not conceive. Consequently, with the 5 beliefs, 
   he/she does not open the box. However, 
   this is at odds with reality: 
 any sensible game player should have felt $p_9$, 
 even if his/her original perception about 
 the key was $p_1$; about the box was $p_4$ and $p_5$; 
 and about the note was $p_7$. {\it Treating this case
     with latent beliefs}: 
 Let us suppose the following structures 
 about the items: {\small $\text{Of the key} := (\{p_1\}, 
     \{p_1(p_2, p_3)\})$}, {\small $\text{Of the box} := 
 (\{p_4 \wedge p_5\}, \{{p_4 \wedge p_5}(p_8,p_6), 
     {p_4 \wedge p_5}(p_3, p_{10}), 
     {p_4 \wedge p_5}(p_9, p_2) 
     \})$}, and
     {\small $\text{Of the note} := (\{p_7\}, \{p_7(p_5,p_8)\})$}. 
     The player accepts these in this order, which 
     means that he/she consciously sees $p_1$ 
     about the key and $p_4$ and $p_5$ about the box. 
     About the note, he/she consciously sees 
     $p_7$ and $p_8$ (the latter is visible 
     because $p_5$ is visible). But $p_8$ 
     lets $p_6$ about the box to come to his/her consciousness. 
    Both $p_6$ and $p_8$ being visible, a sensible player 
    surmises $p_9$. And when it is added to his/her belief 
    set, he/she discovers
    $p_2$ about the box. But $p_2$ helps 
    him/her to conceive $p_3$ about the key, which then 
    reveals that he/she has had $p_{10}$ (about the box), as required. 
 {\it The second problem 
     (concerning dependent beliefs)}: 
 Suppose that a player
 conceives $p_1, p_2 \supset p_3$ (in Town B), 
 $p_4, p_5, p_6$ (in Town A), and $p_7, p_8, p_9$ (in Town C). 
 These should be enough to discover $p_2$ and to have $p_{10}$ 
 subsequently. 
 But suppose that he/she drops the key before getting to 
 Town C. Then his/her belief must be revised 
 with $\neg p_1$. The revision will drop $p_1$. However,
 under the AGM belief theory, the revision 
 does not drop $p_2 \supset p_3$ (Cf. Recovery in 
 Figure \ref{postulates}). 
  So in Town C, he/she has  
 $\neg p_1, p_2 \supset p_3, p_4, p_5, p_6, p_7, p_8$ and $p_9$.
 He/she presses the buttons according to $p_9$ and 
 gains $p_2$. Then $p_2$ and $p_2 \supset p_3$ 
 produces $p_3$. But because he/she has dropped 
 the key, we must ask here; what key? There is no key. 
 The problem is that `the key' and `the key hole' 
 in $p_3$ 
 are presupposing the existence of some key and some 
 key hole. It is generally difficult to recognise 
 the structure among propositions in propositional logic. 
 We could avoid this inconvenience by 
 replacing $p_3$ with $p_x$: If I have a key 
 and if I have a key hole, then I apply the key 
 to the key hole. But such replacement 
 would effectively disallow some natural beliefs 
 of the kind of $p_3$ 
 to appear as a belief. Also, in cases 
 where it is hard to think of $p_{\alpha}$ 
 (that is, not of whether it is true, 
 but of it)
 from one's existing knowledge/belief, 
 it is very unnatural to presume that 
 he/she would nonetheless
 come up with $p_{\alpha} \supset 
 p_{\beta}$. {\it 
     With latent beliefs}: 
 Assume the same belief structures about the items as before. 
 When the player drops the key and revises his/her 
 belief set with $\neg p_1$, by Observation 1 (Adequacy),
 the associated $p_1(p_2, p_3)$ is also dropped, which 
 precludes the stated problem in this  
 setting.  
 \hide{ 
 \indent For another example, 
 consider the following 
 propositions.  
 $\{p_1\}^\diamond$: {\it  
     Anticoagulant is prescribed. 
 }  and 
 $\{p_2\}^\diamond$: {\it  
     Emma is diagnosed deep vein thrombosis.} 
 Suppose the following propositions. 
\begin{multicols}{2}  
    {\small 
  \begin{enumerate}[leftmargin=0.4cm]
      \item $p_1$: I prescribe anticoagulant 
          (warfarin, unfractionated heparin, ...)
          to Emma.  
      \item $p_2$: Warfarin is a suitable choice 
          to Emma.  
      \item $p_3$: My patient is Emma.  
      \item $p_4$: Emma is in the first 
          trimester of pregnancy. 
      \item $p_5$: Emma is diagnosed 
          deep vein thrombosis. 
      \item $p_6$: Warfarin is contraindicated 
          in pregnancy.  
     \end{enumerate}    
 }
 \end{multicols}
 \noindent Suppose ${\{p_1\}^\diamond = 
     (\{p_1\}, \{p_1(p_6, \neg p_2))}$, 
     ${\{p_3\}^\diamond = 
         (\{p_2\}, \emptyset)}$, 
     and ${\{p_4\}^\diamond = 
         (\{p_4\}, \emptyset)}$.  
     Consider a doctor who believes 
     $\{p_1\}^\diamond$ and $\{p_2\}^\diamond$, 

 }
 {\ }\\\indent We have presented a new perception model 
 that incorporates latent beliefs. Expansion of 
 our work with suitable postulates 
 on iterated belief revision is one obvious direction. 
 Studies into reasonable resolution of the new 
 kind of inconsistencies should be also interesting.
\bibliographystyle{named} 
\bibliography{references}    
\end{document}